%% file: neuralga_paper.tex
\documentclass[twoside]{article}

\usepackage{amsmath}
\usepackage{amsfonts}
\usepackage{amsthm}
\usepackage[utf8]{inputenc}

\usepackage{neuralga_format}
\usepackage{verbatim}
\usepackage{graphicx}
\usepackage{xcolor}
\usepackage{tikz-cd}

\input{preamble}

\begin{document}

\twocolumn[
  \ourtitle{Neural Group Actions}
  \ourauthor{Span Spanbauer \And Luke Sciarappa}
  \ouraddress{MIT \And MIT}

]


\begin{abstract}

We introduce an algorithm for designing \textit{Neural Group Actions}, collections of deep neural network architectures which model symmetric transformations satisfying the laws of a given finite group. This generalizes involutive neural networks $\mathcal{N}$, which satisfy $\mathcal{N}(\mathcal{N}(x))=x$ for any data $x$, the group law of $\mathbb{Z}_2$. We show how to optionally enforce an additional constraint that the group action be volume-preserving. We conjecture, by analogy to a universality result for involutive neural networks, that generative models built from \textit{Neural Group Actions} are universal approximators for collections of probabilistic transitions adhering to the group laws. We demonstrate experimentally that a \textit{Neural Group Action} for the quaternion group $Q_8$ can learn how a set of nonuniversal quantum gates satisfying the $Q_8$ group laws act on single qubit quantum states.

\end{abstract}

\section{Introduction}

Symmetry is ubiquitous, appearing throughout nature, art, mathematics, and the sciences. A great deal of effort has been put into designing neural networks which are effective at modeling \textit{symmetric data}, that is, designing networks which are invariant or equivariant to various symmetries.

We address a distinct problem, the problem of modeling \textit{symmetric transformations}. We show how to design neural networks which satisfy the laws of a symmetry. For example, the group capturing the symmetry of reflection is called $\mathbb{Z}_2$. It has the group law $r^2 = \mathrm{id}$, that is, applying a reflection twice does nothing. We show how to design a neural network $\mathcal{N}$ which satisfies this group law, so that $\mathcal{N}(\mathcal{N}(x)) = x$ for any data $x$, as well as how to design neural networks which satisfy the group laws for any other finite group.

These \textit{Neural Group Actions} can be used directly to model symmetric transformations which appear throughout the sciences, particularly in quantum mechanics and statistical physics. They can also be used to neurally accelerate classical algorithms which involve choosing functions satisfying algebraic constraints taking the form of group laws.

\newpage
This second-use case has already been demonstrated. Involutive Markov Chain Monte Carlo (MCMC)~\citep{neklyudov2020involutive} is an approach to designing valid MCMC proposals based on the fact that generating proposals using an involutive function, that is, a function satisfying the $\mathbb{Z}_2$ group laws, makes calculating the acceptance ratio tractable. \citet{spanbauer2020deep} introduced a class of neural networks which exactly satisfy this $\mathbb{Z}_2$ group law, called involutive neural networks, and used them in the context of Involutive MCMC to design \textit{Involutive Neural MCMC}, a fast neural MCMC algorithm.

We present \textit{Neural Group Actions}, a broad generalization of involutive neural networks which are useful for modeling symmetric transformations and for neurally accelerating algorithms.

\noindent {\bf Contributions.} This paper shows, for any finite group, how to design neural networks which exactly satisfy its group laws. Specifically, it presents the following contributions:
\vspace*{-5pt}
\begin{enumerate}
\item This paper introduces an algorithm for designing \textit{Neural Group Actions}. For any given finite group, we show how to design a high-capacity deep neural network architecture for each group element. Collectively, these neural networks exactly satisfy the group laws for any setting of the network parameters. All of these neural networks share the same parameters, but differ in structure.
\item This paper describes how to optionally constrain \textit{Neural Group Actions} to exactly preserve volume, that is, to constrain the determinant of the Jacobian of these networks to have magnitude 1.
\item This paper conjectures, by anology to a previous result~\citep{spanbauer2020deep}, that generative models built from \mbox{\textit{Neural~Group~Actions}} are universal approximators for collections of probabilistic transitions.
\item This paper demonstrates experimentally that a \textit{$Q_8$--Neural Group Action} can learn how the $R_x(\pi)$, $R_y(\pi)$, and $R_z(\pi)$ gates, which act according to the quaternion group $Q_8$ group laws, transform single qubit states.
\end{enumerate}

\newpage

\section{Background}
\label{sec:background}

Symmetries are expressed mathematically as groups---each type of symmetry has a corresponding group. For example, simple bilateral symmetry is expressed by the two element group $\mathbb{Z}_2$, the symmetries of an icosahedron are expressed by the 120 element group $A_5 \times \mathbb{Z}_2$, and the symmetries of spacetime in special relativity are expressed by a Lie group called the Poincar\'e group~\citep{kim2012theory}; this group has infinitely many elements.

Symmetric transformations are expressed mathematically as the action of a group on a set. For example, reflection across some axis and rotation by 180 degrees are distinct group actions of $\mathbb{Z}_2$ on the plane. To qualify as a group action of a group $G$, a set of operations must satisfy the defining laws of $G$ under composition. The group law of $\mathbb{Z}_2$ is $r^2 = \mathrm{id}$; both reflection and rotation by 180 degrees satisfy this group law since composing each operation with itself yields the identity operation.

In this paper we show, for any finite group $G$, how to design a \textit{$G$--Neural Group Action}: a set of neural networks which exactly satisfy $G$'s group laws by construction, independent of the network's parameters.

\section{Neural Group Actions}
\label{sec:neuralga}

\begin{figure*}[!t]

\centering 
\vspace{1mm}
\begin{tikzcd}[column sep=3mm]
                && \boxed{T^{-1}_e} \arrow[rrr,""] &&& \boxed{T_e} \arrow[rdd,""] && \\[-15pt]
                && \boxed{T^{-1}_a} \arrow[rrr,""] &&& \boxed{T_a} \arrow[rd,""] && \\[-28pt]
 x \arrow[r,""]  & \boxed{H^{-1}} \arrow[ru,"",end anchor=west] \arrow[ruu,"",end anchor=west] \arrow[rd,"",end anchor=west] \arrow[rdd,"",end anchor=west]        &&&&&\boxed{H} \arrow[r,""] &e \cdot x\\[-28pt]
                && \boxed{T^{-1}_b} \arrow[rrr,""] &&& \boxed{T_b} \arrow[ru,""] && \\[-15pt]
                && \boxed{T^{-1}_c} \arrow[rrr,""] &&& \boxed{T_c} \arrow[ruu,""] && 
\end{tikzcd}
\hspace{5mm}
\begin{tikzcd}[column sep=3mm]
                && \boxed{T^{-1}_e} \arrow[rrrd,"",in=160,out=-20,looseness=0.5] &&& \boxed{T_e} \arrow[rdd,""] && \\[-15pt]
                && \boxed{T^{-1}_a} \arrow[rrru,"",in=200,out=20,looseness=0.5] &&& \boxed{T_a} \arrow[rd,""] && \\[-28pt]
 x \arrow[r,""]  & \boxed{H^{-1}} \arrow[ru,"",end anchor=west] \arrow[ruu,"",end anchor=west] \arrow[rd,"",end anchor=west] \arrow[rdd,"",end anchor=west]        &&&&&\boxed{H} \arrow[r,""] &a \cdot x\\[-28pt]
                && \boxed{T^{-1}_b} \arrow[rrrd,"",in=160,out=-20,looseness=0.5] &&& \boxed{T_b} \arrow[ru,""] && \\[-15pt]
                && \boxed{T^{-1}_c} \arrow[rrru,"",in=200,out=20,looseness=0.5] &&& \boxed{T_c} \arrow[ruu,""] && 
\end{tikzcd}

\vspace{6mm}
\begin{tikzcd}[column sep=3mm]
                && \boxed{T^{-1}_e} \arrow[rrrddd,"",in=160,out=-20,looseness=0.5] &&& \boxed{T_e} \arrow[rdd,""] && \\[-15pt]
                && \boxed{T^{-1}_a} \arrow[rrrddd,"",in=160,out=-20,looseness=0.5] &&& \boxed{T_a} \arrow[rd,""] && \\[-28pt]
 x \arrow[r,""]  & \boxed{H^{-1}} \arrow[ru,"",end anchor=west] \arrow[ruu,"",end anchor=west] \arrow[rd,"",end anchor=west] \arrow[rdd,"",end anchor=west]        &&&&&\boxed{H} \arrow[r,""] &b \cdot x\\[-28pt]
                && \boxed{T^{-1}_b} \arrow[rrruuu,"",in=200,out=20,looseness=0.5] &&& \boxed{T_b} \arrow[ru,""] && \\[-15pt]
                && \boxed{T^{-1}_c} \arrow[rrruuu,"",in=200,out=20,looseness=0.5] &&& \boxed{T_c} \arrow[ruu,""] && 
\end{tikzcd}
\hspace{5mm}
\begin{tikzcd}[column sep=3mm]
                && \boxed{T^{-1}_e} \arrow[rrrdddd,"",in=160,out=-20,looseness=0.5] &&& \boxed{T_e} \arrow[rdd,""] && \\[-15pt]
                && \boxed{T^{-1}_a} \arrow[rrrdd,"",in=158,out=-20,looseness=0.5] &&& \boxed{T_a} \arrow[rd,""] && \\[-28pt]
 x \arrow[r,""]  & \boxed{H^{-1}} \arrow[ru,"",end anchor=west] \arrow[ruu,"",end anchor=west] \arrow[rd,"",end anchor=west] \arrow[rdd,"",end anchor=west]        &&&&&\boxed{H} \arrow[r,""] &c \cdot x\\[-28pt]
                && \boxed{T^{-1}_b} \arrow[rrruu,"",in=202,out=20,looseness=0.5] &&& \boxed{T_b} \arrow[ru,""] && \\[-15pt]
                && \boxed{T^{-1}_c} \arrow[rrruuuu,"",in=200,out=20,looseness=0.5] &&& \boxed{T_c} \arrow[ruu,""] && 
\end{tikzcd}

\vspace{4mm}

\caption{Architecture for a \textit{$K_4$--Neural Group Action}, taking $S=G$ and using left-multiplication as the action of $G$ on $S$. The Klein four-group $K_4$ has four elements $\{e,a,b,c\}$ satisfying the group laws $\langle \: a^2 = b^2 = c^2 = e, \: ab = c \: \rangle$. We obtain four different neural networks, one implementing the action of each group element on $\mathbb{R}^{p|S|}$. These four neural networks share the same parameters, but are structured differently; the permutation in the center of each network is the permutation obtained by $g$'s action on $S$.  These neural networks collectively satisfy the group laws under composition, so for example $(a \cdot (a \cdot x))$ implements the identity function; by following this diagram one can easily verify this by hand.}
\label{fig:klein}

\end{figure*}
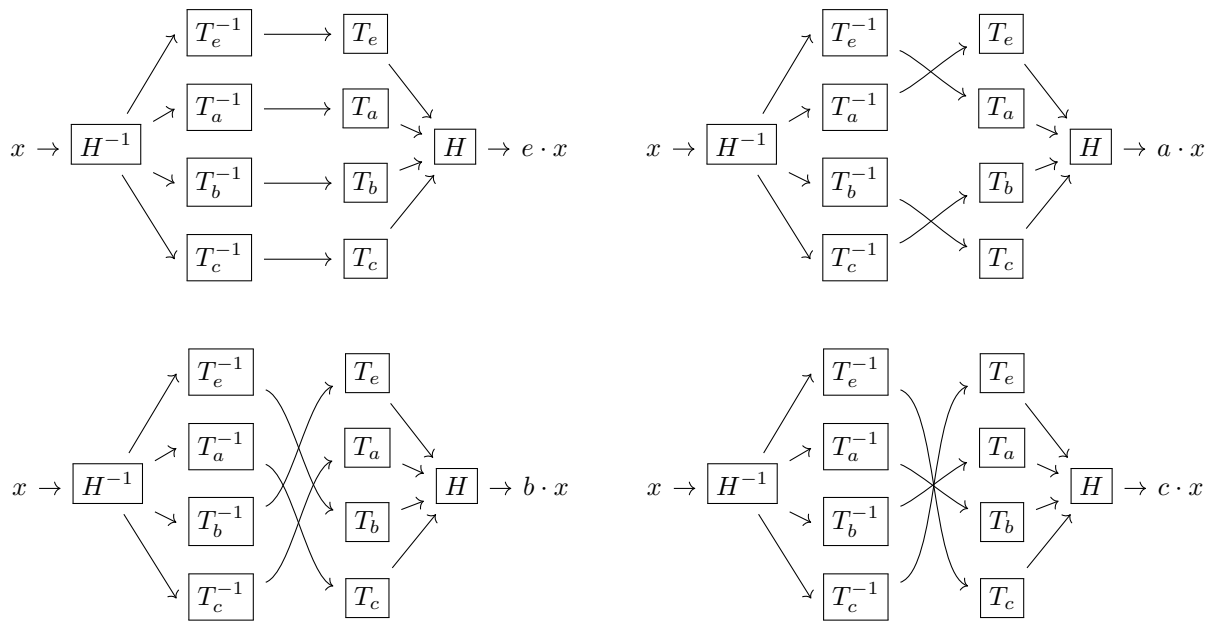
When training a \textit{$G$--Neural Group Action}, we'll want to optimize an objective function defined on the space of actions of $G$ on $\mathbb{R}^p$, for some $p$.
In order to explore this space, we need to parameterize it, that is, define a coordinate system on it. This is what a neural network architecture is: a mapping from the network's parameters, the coordinates, to a point in some interesting space of functions.

The usual deep learning methods give useful parameterizations of simpler spaces such as the space of functions from $\mathbb{R}^{p}$ to $\mathbb{R}^{q}$ or the space of invertible functions from $\mathbb{R}^{p}$ to $\mathbb{R}^{p}$~\citep{ardizzone2018analyzing}. We will use these parameterizations of simpler classes of functions to build a parameterization of the space of group actions of $G$ on $\mathbb{R}^p$. 

For example, the space of actions of $\mathbb{Z}_2$ on $\mathbb{R}^p$ can be described as the subset of the space of functions $f : \mathbb{R}^p \to \mathbb{R}^p$ consisting of the functions satisfying the equation $f(f(x)) = x$ for all $x \in \mathbb{R}^p$. A \textit{$\mathbb{Z}_2$--Neural Group Action} is a parameterization of this class of functions. We will show how to design an architecture that exactly satisfies this group law (or any other group law) due to its structure alone. 

If we were to instead attempt to explore this space of group actions using an unstructured neural network parameterizing the space of functions $\mathbb{R}^{p}$ to $\mathbb{R}^{p}$, we would have great difficulty enforcing this group law. If we were to use an unstructured neural network, even if we started out with the coordinates of a function $f$ that satisfied this equation, the direction that gradient descent moves $f$ in to improve the objective function may take us out of that subset. These coordinates would be too expressive; they would allow us to specify more functions than the ones of interest. We now show how to construct specially structured neural networks which are not too expressive; they express only valid group actions.

\subsection{Construction for \textit{Neural Group Actions}}

Suppose we have an action of $G$ on a finite set $S$, which we'll write as $(g, s) \mapsto gs$. Functions from $S$ to $\mathbb{R}^p$ can be viewed as vectors of length $p|S|$. We'll denote the whole vector by $x \in \mathbb{R}^{p|S|}$, and its value on $s \in S$ by $x_s \in \mathbb{R}^p$. The coordinates we use will be families of invertible functions $T_s : \mathbb{R}^p \to \mathbb{R}^p$ indexed by $s \in S$. We can parameterize these invertible functions, in turn, using existing invertible neural network architectures~\citep{ardizzone2018analyzing} such as NICE~\citep{dinh2014nice} or i-RevNet~\citep{jacobsen2018revnet}.
\begin{lemma}
\label{lem:groupact}
Given an action of a group $G$ on a finite set $S$ and an arbitrary invertible function $T_s: \mathbb{R}^p \to \mathbb{R}^p$ for each $s \in S$, we can define an action of $G$ on $\mathbb{R}^{p|S|}$ by setting 
\[
(g\cdot x)_{s} = T_s(T_{g^{-1}s}^{-1}(x_{g^{-1}s}))
\] for all $g \in G$, $x \in \mathbb{R}^{p|S|}$, and $s \in S$.
\end{lemma}
Note: the above formula is equivalent to demanding that
\[
(g\cdot x)_{gs} = T_{gs}(T_{s}^{-1}(x_s)),
\] a form which may be more enlightening.
\begin{proof}
We need only check that the purported action respects the identity and the multiplication of $G$. If $e$ is the identity element of $G$, we need $e\cdot x = x$, which we may check coordinatewise:
\[
(e\cdot x)_s = T_s(T_{e^{-1}s}^{-1}(x_{e^{-1}s})) = T_s(T_s^{-1}(x_s)) = x_s,
\] since $e^{-1}s = es = s$. Now if we have $g, h$ in $G$, we also need that $(gh)\cdot x = g\cdot (h\cdot x)$, which we again may check by comparing their values for each $s$. On the one hand,
\begin{align*}
((gh)\cdot x)_s &= T_{s}(T_{(gh)^{-1}}^{-1}(x_{(gh)^{-1}s})) \\
&= T_{s}(T_{h^{-1}g^{-1}}^{-1}(x_{h^{-1}g^{-1}s})),
\end{align*} since $(gh)^{-1} = h^{-1}g^{-1}$; on the other hand,
\begin{align*}
(g\cdot(h\cdot x))_s &= T_s(T_{g^{-1}s}^{-1}((h\cdot x)_{g^{-1}s})) \\
&= T_s(T_{g^{-1}s}^{-1}(T_{g^{-1}s}(T_{h^{-1}g^{-1}s}^{-1}(x_{h^{-1}g^{-1}s}))))
\end{align*} which is equal to the preceding expression because $T_{g^{-1}s}^{-1}(T_{g^{-1}s}(y)) = y$.
\end{proof}
In general, one may choose $S=G$ and use the action of $G$ on itself by left multiplication as the action required in the hypothesis of Lemma~\ref{lem:groupact}. We take this course in our experiment in Sec.~\ref{sec:experiment}. Since this action is the free $G$-action on a single element, it is in some sense the simplest $G$-action, but also the most expressive: every $G$-action is a sum of quotients of it. It therefore represents a natural architectural choice.

This achieves our goal of parameterizing a part of the space of actions of $G$ on $\mathbb{R}^{p|S|}$ without going outside that space. 
However, though we have avoided being too expressive, we may now wonder if our parameterization is insufficiently expressive. It does have one clear deficiency: it privileges a specific decomposition of $\mathbb{R}^{p|S|}$ into $|S|$ copies of $\mathbb{R}^p$. 

For example, consider $\mathbb{Z}_2 = \langle r \mid r^2 = e\rangle$ with its nontrivial action on $\{0,1\}$. If we take $T_0(x) = 2x$ and $T_1(x) = x$, then our construction above works out to saying that $r \cdot (x, y) = (2y, x/2)$.
However, we could define a similar action that swaps and scales $x + y$ and $x - y$ analogously, that is, an action such that whenever \[r \cdot_\mathrm{new} (x, y) = (x', y'),\] then $x' + y' = 2(x - y)$ and $x' - y' = (x+y)/2$. Solving these, we see that the new action is given by \[r \cdot_\mathrm{new} (x, y) = \left(\frac{5x - 3y}{4}, \frac{3x-5y}{4}\right).\] No choice of $T_0$ and $T_1$ could result in this action under our construction, since both halves of the result vector depend on both halves of the input vector. However, this new action is intuitively doing something very similar to the previous swap-and-halve-or-double action. It is just using $x + y$ and $x - y$ to describe $\mathbb{R}^{2p}$, rather than $x$ and $y$ --- a different but equally valid choice of two $p$-dimensional vectors.

We model this by introducing a conjugation. In addition to the $T_s$, we assume an arbitrary invertible function $H : \mathbb{R}^{p|S|} \to \mathbb{R}^{p|S|}$, and define \[\Tilde{g \cdot
x} = H(g \cdot H^{-1}(x))\] where $-\cdot-$ is the action defined from the $T_s$ as above. This is still a group action. We can parameterize $H$ in the same way as the $T_s$. Including the conjugation doesn't truly get rid of the decomposition of $\mathbb{R}^{p|S|}$ into $|S|$ copies of $\mathbb{R}^{p}$, but it does allow the network to learn the best choice of this decomposition. Although the example above was linear, with \mbox{$H(x,y) = (x+y,x-y)$}, we can have $H$ nonlinear as well, corresponding to curved rather than linear coordinates. 

An example of the full construction of a \textit{Neural Group Action}, including the conjugation, is shown in Fig.~\ref{fig:klein}. 

\subsection{Volume-preserving \textit{Neural Group Action}}

It may be desirable to further constrain oneself to a class of volume preserving functions, that is, those with a Jacobian whose determinant has magnitude 1. \citet{spanbauer2020deep} used generative models derived from volume preserving involutive neural networks in order to make valid Metropolis-Hastings proposals in the context of a neural MCMC algorithm; the networks being volume-preserving enable fast computation of acceptance ratios.

We generalize this prior work by showing how to construct volume-preserving $G$--\textit{Neural Group Actions}.

\begin{lemma}
If $H$ and all of the $T_s$ are volume preserving, then so is the resulting group action.
\end{lemma}
\begin{proof}
If $J_{H,x}$ is the Jacobian of $H$ at $x$, $J_{g,x}$ is the Jacobian of the original action at $x$, and $\tilde{J}_{g,x}$ is the Jacobian of the conjugated action at $x$, then 
\[
\tilde{J}_{g,x} = J_{H, g\cdot H^{-1}(x)}J_{g,H^{-1}(x)}(J_{H,H^{-1}(x)})^{-1}
\] and so if $\operatorname{det}(J_{H,x}) = 1$ for all $x$, $\operatorname{det}(\tilde{J}_{g,x}) = \operatorname{det}(J_{g,H^{-1}(x)})$ for all $x$. Thus it suffices to show that the original action is volume-preserving.

We observe that $J_{g,x}$ is a block permutation matrix; it can be viewed as an $|S|$-by-$|S|$ array of $p$-by-$p$ blocks, such that for each $s$ exactly one block in the corresponding row (or equivalently column) of the array is nonzero. The nonzero block in row $s$ is in column $g^{-1}s$ and is equal to the Jacobian of $T_sT_{g^{-1}s}^{-1}$ at $x_{g^{-1}s}$. By assumption, these Jacobians all have unit determinant. By exchanging rows to move each block into an equal row and column, we see that the determinant of $J_{g,x}$ is equal to $\pm 1$ times the determinant of a block diagonal matrix in which each block has determinant $1$. This can in turn be seen to have determinant $1$ by repeated application of the well-known fact 
\[
\operatorname{det} \begin{pmatrix} A & 0 \\ 0 & B \end{pmatrix} = \operatorname{det}(A) \operatorname{det}(B) 
\] for $A, B$ square blocks of arbitrary (possibly distinct) size. Therefore, $J_{g,x}$ itself has determinant equal to $\pm 1$, i.e.~the original action is volume-preserving, as desired.
\end{proof}

\subsection{Generative models derived from a \textit{Neural~Group~Action}}

As previously mentioned, while \textit{Neural Group Actions} are reasonably expressive, they may not be sufficiently expressive to represent certain group actions---they may not be a universal approximator for group actions as we increase the capacity of the constituent networks $H$ and $T_{s \in S}$. If we find that \textit{Neural Group Actions}, as currently defined, cannot represent certain group actions, we could look for alternative parameterizations that increase the representational power, eventually aiming to prove a universality result.

An alternative approach, motivated by a universality result from \citet{spanbauer2020deep}, is to define a more flexible domain-specific model using \textit{Neural Group Actions} as a building block, and then prove an appropriate domain-specific universality result. 

By analogy to the prior work, we define generative models of probabilistic transitions of a state $\phi : \mathbb{R}^n$ by introducing auxiliary random variables $\pi : \mathbb{R}^m$ such that a deterministic neural group action on the enlarged state space $(\phi,\pi) : \mathbb{R}^{n+m}$ defines a set of probabilistic transitions $\phi \xmapsto{g} \phi'$. We conjecture that these \textit{Generative Neural Group Actions} are universal approximators of sets of probabilistic transitions, generalizing the previous result obtained for involutive generative models~\citep{spanbauer2020deep}.

We now move to discuss experimental results showing that a \textit{Neural Group Action} can learn a group action acting on quantum states arising from the composition of certain single-qubit gates.

\section{Experiment}
\label{sec:experiment}

In quantum computing, nonuniversal sets of quantum logic gates generate, by composition, sets of transformations with the structure of finite groups. 

For example, the single-qubit gates $R_x[\pi]$, $R_y[\pi]$, and $R_z[\pi]$ generate eight unique single-qubit transformations. These transformations satisfy the laws of the quaternion group $Q_8$:
\begin{align*}
&R_x[\pi]^2 = R_y[\pi]^2 = R_z[\pi]^2 = R_x[\pi]R_y[\pi]R_z[\pi]\\ \\
&(R_x[\pi]R_y[\pi]R_z[\pi])^2 = id.
\end{align*}
In this experiment we learn the action of these transformations on quantum states via supervised training of a \textit{$Q_8$--Neural Group Action}.

\begin{figure}[!t]
{\centering

\vspace{1mm}
\textbf{Learned action on a random single qubit state}

\vspace{7mm}
\hspace{4mm}
\begin{picture}(200,200)
\put(0,0){\includegraphics[width=70mm]{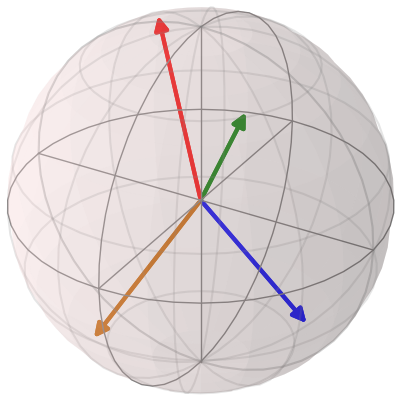}}

\put(95,-8){$|1\rangle$}
\put(95,201){$|0\rangle$}
\put(40,48){$x$}
\put(195,68){$y$}
\end{picture}
}
\vspace{5mm}
\caption{A \textit{$Q_8$--Neural Group Action} was trained to perform the eight single qubit transformations generated by the $R_x(\pi)$ and $R_y(\pi)$ quantum gates. A random quantum state is shown in green on this Bloch sphere~\citep{nielsen2002quantum} along with each of the states resulting from the learned transformation. Rendering was done in the QuTiP quantum dynamics framework~\citep{johansson2012qutip}. While all eight resulting states were drawn on the Bloch sphere, only four physically distinct states appear since four pairs of transformed states are related by a physically irrelevant global phase. These transformed states precisely match the result of applying the true single qubit gates, to seven digits of accuracy.}
\label{fig:bloch}
\end{figure}
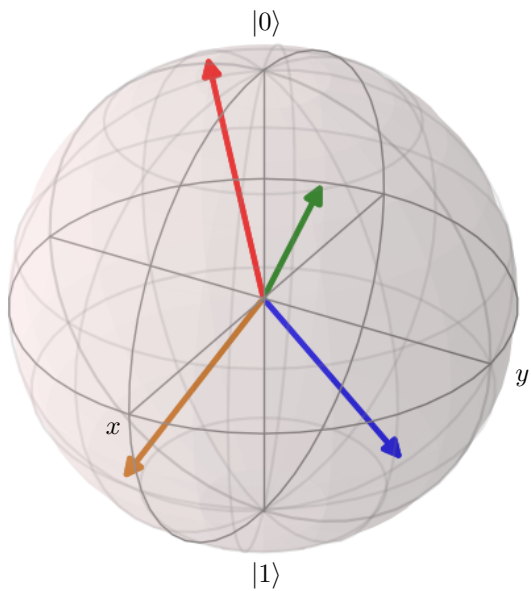

\subsection{Architecture and training}

To build our \textit{$Q_8$--Neural Group Action}, several architectural choices were made which affect the capacity of the network.
\begin{itemize}
\item The dimension of each invertible transformation $T_s : \mathbb{R}^p \to \mathbb{R}^p$ was chosen to be $p=16$.
\item Each invertible transformation $T_s : \mathbb{R}^p \to \mathbb{R}^p$ was parameterized using a single NICE additive coupling layer~\citep{dinh2014nice}.
\item The invertible conjugation $H : \mathbb{R}^{p|S|} \to \mathbb{R}^{p|S|}$ was parameterized with three cascaded additive coupling layers~\citep{dinh2016density}, with random permutations interspersed to improve mixing~\citep{ardizzone2018analyzing}.
\end{itemize}

We performed supervised training of the \textit{$Q_8$--Neural Group Action}, minimizing the sum of squared errors of each of the real and imaginary components of each element of each transformed quantum state. We trained to convergence in PyTorch~\citep{NEURIPS2019_9015} using an Adam optimizer~\citep{kingma2014adam}, which took about two minutes on an NVIDIA~RTX~2080~Ti GPU.

\subsection{Results}

Our \textit{$Q_8$--Neural Group Action} learned to perform these eight single-qubit operations to high accuracy, with a $L_2$ loss of about $3 \cdot 10^{-13}$, that is, about 6 digits of accuracy in each component of each quantum state vector. An example of this set of transformations is shown in Fig.~\ref{fig:bloch}

Furthermore, regardless of whether the parameters were randomly initialized or fully trained, the \textit{$Q_8$--Neural Group Action} satisfied the $Q_8$ group laws to nearly the limits of floating point precision.

\section{Discussion}

This paper has, for every finite group $G$, shown how to design \textit{$G$--Neural Group Actions} parameterizing a large class of group actions of $G$. It has shown experimentally that this method works for realistic group actions by moderate-sized groups; a \textit{$Q_8$--Neural Group Action} was capable of accurately learning a group action naturally arising from the composition of single qubit gates. Furthermore, these networks are practical to train. Their depth remains constant for large groups, and their evaluation time scales only linearly with group size. Training took only minutes for our experimental demonstration.

We have also introduced two variations on this method. The first are \textit{Volume Preserving Neural Group Actions} which parameterize the subset of group actions that have a Jacobian whose determinant has magnitude 1. The second are \textit{Generative Neural Group Actions} which we conjecture can universally approximate sets of probabilistic transitions. These variations generalize methods used by \citet{spanbauer2020deep} in a neurally accelerated MCMC algorithm. We expect that these variations may have a role to play in neurally accelerating other probabilistic algorithms.

\section{Acknowledgements}

The authors would like to thank Fonterra for supporting this research.

\bibliography{mybib}

\end{document}

%% file: preamble.tex
\usepackage[numbers]{natbib}
\usepackage{microtype}

\usepackage[hide=false,setmargin=true,marginparwidth=0.5in]{marginalia}

\usepackage{multicol}
\usepackage{}

\usepackage{graphicx} 
\usepackage{subfigure}

\usepackage{url}
\usepackage{bm}
\usepackage{framed}

\usepackage{amsmath}
\usepackage{amsfonts}
\DeclareMathAlphabet{\mathpzc}{OT1}{pzc}{m}{it}

\usepackage{stackrel}

\usepackage[shortlabels]{enumitem}

\usepackage{xcolor}
\usepackage{tikz}

\usepackage{rotating}

\usepackage{amsthm}
\usepackage{mathtools}

\usepackage{algorithm}
\usepackage{algpseudocode}

\usepackage[overlay]{textpos}

\theoremstyle{theorem}
\newtheorem{theorem}{Theorem}[section]

\theoremstyle{lemma}
\newtheorem{lemma}[theorem]{Lemma}

\theoremstyle{definition}

\theoremstyle{conjecture}

\usetikzlibrary{calc,trees,positioning,arrows,chains,shapes.geometric,%
    decorations.pathreplacing,decorations.pathmorphing,shapes,%
    matrix,shapes.symbols}

\tikzset{
>=stealth',
  punktchain/.style={
    rectangle, 
    rounded corners, 
    draw=black, very thick,
    text width=10em, 
    minimum height=3em, 
    text centered, 
    on chain},
  line/.style={draw, thick, <-},
  element/.style={
    tape,
    top color=white,
    bottom color=blue!50!black!60!,
    minimum width=8em,
    draw=blue!40!black!90, very thick,
    text width=10em, 
    minimum height=3.5em, 
    text centered, 
    on chain},
  every join/.style={->, thick,shorten >=1pt},
  decoration={brace},
  tuborg/.style={decorate},
  tubnode/.style={midway, right=2pt},
}

\usepackage{hyperref}

\usepackage{dblfloatfix}
\usepackage{float}
\usepackage{inconsolata}

\usepackage{hyperref}
\hypersetup{
  colorlinks=true,
  citecolor=blue,
  urlcolor=magenta,
}
